\pdfoutput=1

\documentclass[10pt,journal,compsoc]{IEEEtran}
\ifCLASSOPTIONcompsoc
  \usepackage[nocompress]{cite}
\else
  \usepackage{cite}
\fi
\usepackage{graphicx}
\usepackage{amsmath}
\usepackage{amsthm}
\usepackage{booktabs}
\usepackage{algorithm}
\usepackage{multirow}
\usepackage{algorithmic}
\usepackage{graphicx}
\usepackage{subfigure}
\usepackage{amsfonts}
\usepackage{amsthm}
\usepackage{multirow}
\usepackage{makecell}
\usepackage{subfigure}

\newtheorem{theorem}{Theorem}
\newtheorem{lemma}{Lemma}
\ifCLASSINFOpdf
\else
\fi
\begin{document}
%
\title{Domain-Class Correlation Decomposition for \\
	Generalizable Person Re-Identification}
\markboth{IEEE TRANSACTIONS ON MULTIMEDIA}%
{Shell \MakeLowercase{\textit{et al.}}: Bare Demo of IEEEtran.cls for Computer Society Journals}
%
\author{Kaiwen~Yang,~Xinmei~Tian,~\IEEEmembership{Member,~IEEE}
	\IEEEcompsocitemizethanks{\IEEEcompsocthanksitem K. Yang and X. Tian are with CAS Key Laboratory of Technology in Geo-spatial Information Processing and Application System, University of Science and Technology of China, Hefei 230027, China\protect\\
		E-mail: kwyang@mail.ustc.edu.cn, xinmei@mail.ustc.edu.cn}}


\IEEEtitleabstractindextext{%
\begin{abstract}
Domain generalization in person re-identification is a highly important meaningful and practical task in which a model trained with data from several source domains is expected to generalize well to unseen target domains. Domain adversarial learning is a promising domain generalization method that aims to remove domain information in the latent representation through adversarial training. However, in person re-identification, the domain and class are correlated, and we theoretically show that domain adversarial learning will lose certain information about class due to this domain-class correlation. Inspired by casual inference, we propose to perform interventions to the domain factor $d$, aiming to decompose the domain-class correlation. To achieve this goal, we proposed estimating the resulting representation $z^{*}$ caused by the intervention through first- and second-order statistical characteristic matching. Specifically, we build a memory bank to restore the statistical characteristics of each domain. Then, we use the newly generated samples $\{z^{*},y,d^{*}\}$ to compute the loss function. These samples are domain-class correlation decomposed; thus, we can learn a domain-invariant representation that can capture more class-related features. Extensive experiments show that our model outperforms the state-of-the-art methods on the large-scale domain generalization Re-ID benchmark.
\end{abstract}

\begin{IEEEkeywords}
Person re-identification, domain generalization, information entropy, casual inference
\end{IEEEkeywords}}

\maketitle

\IEEEdisplaynontitleabstractindextext

%
\IEEEpeerreviewmaketitle

\IEEEraisesectionheading{\section{Introduction}\label{sec:introduction}}

%
%
%
%
\IEEEPARstart{P}{erson} re-identification (ReID) aims to identify a specific person from a set of surveillance cameras over time. ReID plays a significant role in many vision-related applications, e.g., video surveillance, content-based video retrieval, and identification from CCTV cameras. With the development of deep convolutional neural networks (CNNs) \cite{he2016deep}, person ReID methods have achieved remarkable performance in a supervised setting, where the training and testing samples are from the same domain \cite{yang2018local,luo2019strong}. However, these supervised methods learn deep features that naturally
overfit to the training dataset (source domain) and thus suffer severe performance degradation when applied to unseen target domains. As a result, these supervised methods are problematic due to expensive labeling costs and are hardly applicable in practice. To tackle these issues, unsupervised domain adaptation (UDA) methods have been introduced into person ReID \cite{ge2020mutual}. The UDA methods utilize both labeled source domain data and unlabeled target domain data to learn a model that fits the data distribution in the target domain. Although UDA methods alleviate the need to label the target domain, they still require a large amount of target data to update the model.

Compared to UDA, domain generalization (DG) addresses the more challenging setting where we use labeled data from several source domains to train a model and expect the model to generalize well to unseen target domains. Therefore, DG is much more suitable for real-world applications. As pointed out by \cite{li2018domain}, DG aims to learn a domain-invariant feature representation that is also sufficiently discriminative for underlying tasks. Domain adversarial learning (DAL) \cite{ganin2016domain,li2018domain,li2018deep,akuzawa2019adversarial} is  a promising method for achieving this goal. DAL was originally invented for UDA \cite{ganin2016domain} but later was proven to be effective for DG \cite{li2018domain}. DAL introduces a domain discriminator parameterized by deep neural networks and learns domain-invariant feature representation by deceiving it.

However, most DG methods consider a homogeneous setting where different domains share the same class space (Fig. \ref{fig1}(a)). By contrast, person ReID is a heterogeneous setting where different domains have different class spaces (Fig. \ref{fig1}(b)). As shown in Fig. \ref{fig1}(b), $y$ and $d$ become correlated due to some confounders $h$. $h$ could be spatial and temporal deviations when collecting data for different domains. As a result, in person ReID, different domains contain different classes. The conditional entropy satisfies $H(d|y)<H(d)$ and $H(y|d)<H(y)$, and thus, the mutual information satisfies $I(y,d) = H(d)-H(d|y)=H(y)-H(y|d)>0$, clearly indicating the correlation between $y$ and $d$. Therefore, the direct application DAL to person ReID is problematic. Theorem 1 shows the degree of discriminative information loss when we learn domain-invariant feature representation using DAL in person ReID.

\begin{figure*}
	\centering
	\includegraphics[width=16cm]{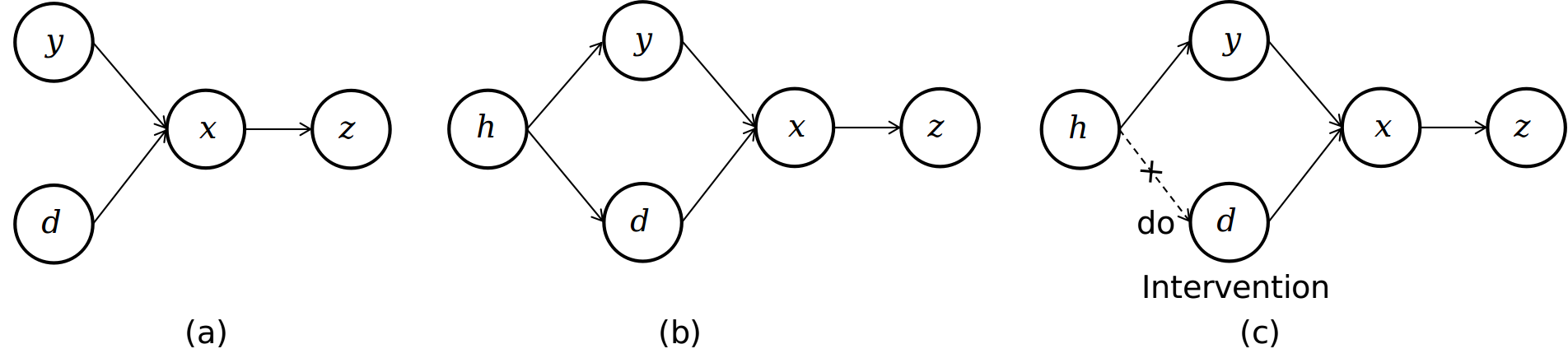}

	\caption{Causal graphs of different situations. $y$ is the class label, $d$ is the domain label, $x$ is the sample, $z$ denotes the embedding and $h$ is the confounder. (a) Homogeneous domain generalization, where $y$ and $d$ are uncorrelated. (b) Domain generalization in person ReID, where $y$ and $d$ have confounders $h$ and thence are correlated. (c) Intervene factor $d$; thus, the backdoor path between $y$ and $d$ is cut off. }
	\label{fig1}

\end{figure*}

Casual inference \cite{pearl2016causal,yao2020survey} reveals that if we perform intervention for one factor, we will remove all of its causes (Fig. \ref{fig1}(c)). In this manner, the back door path  $d\leftarrow h \rightarrow y $ is cut off. Thus, when we evaluate the effect of $d$ on final embedding $z$, it will no longer be disturbed by $y$ through the back door path. Now, the problem lies in how to obtain $z^{*}=z(x^{*})=z(x(y,do(d)))$ if we fix $y$ and perform intervention for $d$, where $*$ denotes the change of the variable value caused by the intervention and $do()$ means intervention. According to back door adjustment formula \cite{pearl2016causal}, we can take the expectation of $h$ to compute $z^{*}$ using observed data. However, as mentioned before, $h$ in person ReID denotes spatial and temporal deviations that are unobservable. Therefore, in this paper, we directly approximate $z^{*}$ according to both the first-order and second-order statistical characteristics of feature maps. \cite{huang2017arbitrary} pointed out that the style of a feature map is represented by the mean of each channel, while \cite{li2017demystifying} further utilized the covariance matrix to represent the feature map style. Inspired by these works, when we $do(d)$, we approximate the resulting $z^{*}$ by transforming the feature map $f^{*}$ to match the first- and second-order statistical characteristics of another domain. To acquire the statistical characteristics of each domain, we build a memory bank to count the mean and covariance of the observed feature map belonging to each domain. Then, we use the generated sample $\{z^{*},y,do(d)\}$ to compute the loss function and update the model. In this way, we decompose the domain-class correlation and thus avoid the disturbance of the class factor when we learn domain-invariant representations.

Our main contributions are summarized as follows:
\begin{itemize}
	\item From the perspective of entropy, we theoretically reveal the potential problem when applying domain adversarial learning to person ReID due to the domain-class correlation.
	\item Inspired by casual inference, we introduce intervention for factor $d$, thereby decomposing the domain-class correlation.
	\item We propose to estimate and obtain new samples $\{z^{*},y,do(d)\}$ caused by intervention through matrix transformation to match the first-order and second-order
	statistical characteristics. We build a memory bank to restore the statistical characteristics of each domain by dynamic updating through the training procedure. We theoretically prove the effectiveness of our method.
	\item  We apply our method to both MobileNet-V2 and ResNet-50 and achieve state-of-the-art results on the large-scale domain generalization ReID benchmark.
	\item We design an extended ablation study to demonstrate the effectiveness of each component of our method.
\end{itemize}



\begin{figure*}[h]
	\centering
	\includegraphics[width=0.95\textwidth]{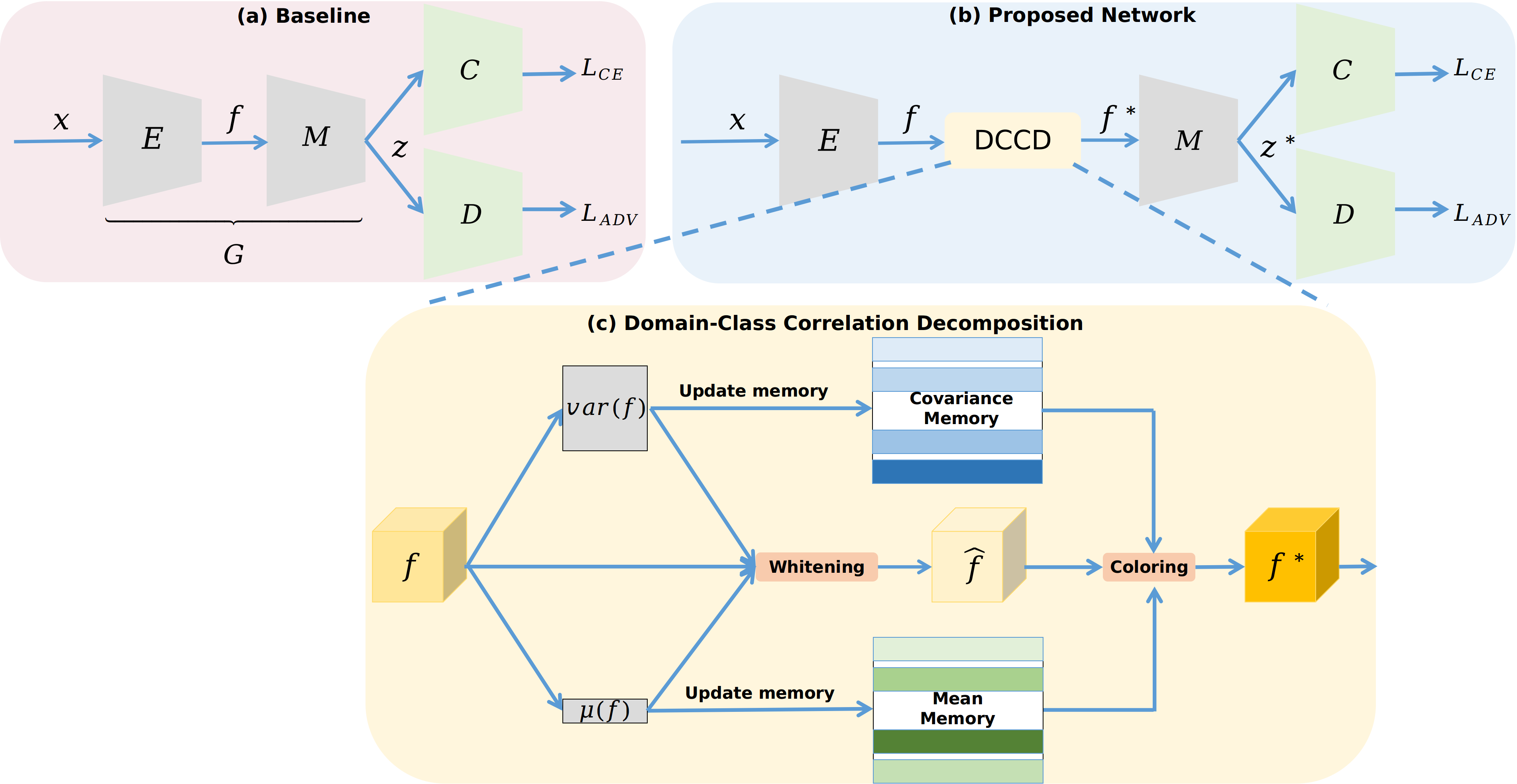}

	\caption{Overall flowchart. (a) Baseline domain adversarial learning. $E$ denotes an encoder network, and $M$ denotes a mapping network. $C$ is a classifier, and $D$ is a domain discriminator. (b) The proposed domain-class correlation decomposition module is inserted into a DAL framework. DCCD generates decomposed new samples to compute the adversarial objective function. (c) Our proposed DCCD module. Two memory banks are introduced to store the domain style information. For a feature map $f$, its own domain style is first removed through whitening and then transferred to another domain style. The outputted $f^{*}$ can have different domain styles; thus, its domain-class correlation is decomposed. }

	\label{fig3}
\end{figure*}

\section{Related work}
\textbf{Domain Generalization.}
Domain generalization methods train a model that is expected to
be generalized to the unseen target by assuming that only multiple source domains are available \cite{ganin2016domain,li2018domain,li2018deep,akuzawa2019adversarial}. For handling this issue, data augmentation is a direct and effective approach for generating diverse data to increase the generalization capability of a model. \cite{zhou2020learning} generate novel domains by maximizing the difference between source and target domains under the constraint of semantic consistency. \cite{qiao2020learning} applies adversarial training to generate challenging examples by a Wasserstein autoencoder (WAE)~\cite{tolstikhin2017wasserstein}. \cite{zhou2021domain} performs Mixup~\cite{zhang2017mixup} in feature space which does not generate raw training data explicitly.

Another type of method aims to learn domain-invariant representations so that the representations can generalize well. Toward this goal, \cite{zhou2020domain, pan2010domain, li2018domain, ghifary2016scatter} explicitly align the data distribution of the two domains. \cite{zhou2020domain} minimizes the Wasserstein distance between different source domains through optimal transport. \cite{pan2010domain} improves the feature transport by proposing adapted transfer component analysis. \cite{li2018domain2} further aligns features from different domains by minimizing the MMD distance of the class-conditional distribution. Domain adversarial learning is another approach for learning domain-invariant representations. \cite{li2018domain} proposes domain adversarial learning that introduces a domain discriminator seeking to classify which domain the representation belongs to. Meanwhile, the generator is trained to extract domain-invariant representations that can deceive the domain discriminator. \cite{li2018deep} further propose building a domain discriminator for each class separately. In this way, they successfully learn domain-invariant representations w.r.t. the joint distribution $P(T(x),y)$. \cite{akuzawa2019adversarial} first considered domain-class correlation and proposed a soft constraint for the domain discriminator. In this paper, we consider domain generalization in person ReID, where domain and class are totally correlated, and we propose a novel framework that can generate domain-class uncorrelated samples.

\textbf{Person Re-Identification.}
Deep learning-based person ReID has recently achieved great success, including but not limited to stripe-based methods~\cite{sun2019learning, wang2018learning, fan2018scpnet，yang2020part, wan2019concentrated}, pose-guided methods~\cite{zhao2017spindle, wei2017glad, sarfraz2018pose} and attention-based methods~\cite{si2018dual, li2018harmonious, li2018diversity}.

To solve the domain shift problem, unsupervised domain adaptation (UDA) has been extensively studied in person ReID. UDA utilizes labeled source domain data and unlabeled target domain data to enhance representation learning. There are two main types of UDA-based ReID methods. The first category of methods transfers the style of the source domain images to that of the target domain images to provide supervision in the target domain~\cite{zou2020joint, zhong2018generalizing, chen2019instance, deng2018image, li2021intra,xie2020progressive, ding2020adaptive}. \cite{deng2018image} proposes a feature-level similarity constraint to better translate the images. The second type of method estimates the target domain pseudo label~\cite{fan2018unsupervised, song2020unsupervised, tang2019unsupervised, yu2019unsupervised}. \cite{fan2018unsupervised}  iteratively by hard clustering. \cite
{fu2019self} further utilizes both global and local features to build multiple independent clusters.

Moreover, domain generalizable person ReID has recently been studied; this method aims to learn discriminative representation with only labeled source domain  data. \cite{song2019generalizable} first defines the large-scale domain generalizable person ReID benchmark with multiple source and target domains. They also propose a domain-invariant mapping network (DIMN) that learns to map representations into domain-invariant space. \cite{jia2019frustratingly} designed a simple and effective backbone for generalizable ReID by inserting instance normalization (IN)~\cite{ulyanov2016instance} into several bottlenecks to remove domain-specific features. \cite{chen2020dual} applies domain adversarial learning to generalizable ReID. They also design several objective functions to bridge the gap between the features from different domains. However, in this paper, we note a problem that may hinder domain adversarial learning in ReID and propose a novel method to solve this problem.

\textbf{Style Transfer.}
Style transfer method takes a content image and a style
image as inputs to synthesize an image with the look of
the former and feel if the latter. \cite{gatys2016image} first proposed gradient descent on the original image through style loss and content loss. \cite{huang2017arbitrary} match the means and variances of deep features between content and style images. \cite{li2017universal} further exploit the
feature covariance matrix instead of the variance by embedding both whitening and coloring processes within a pretrained encoder-decoder module. \cite{li2017demystifying} reveal that matching the covariance of two feature maps is equivalent to minimizing their maximum mean discrepancy (MMD) distance. \cite{li2019learning} accelerate the style transfer procedure by using a neural network to estimate the transformation matrix.

\textbf{Casual Inference.}
\cite{gong2016domain} first propose a causal graph to describe the
generative process of an image as being generated by "domain" and "class". They use this graph for learning invariant components that transfer across domains. \cite{li2018deep,akuzawa2019adversarial}
extend \cite{gong2016domain} with DAL \cite{ganin2016domain}. \cite{heinze2020conditional} propose a similar causal graph for
images while adding auxiliary information such that some images are of the same instance with a different ``style''. This allows the model to identify core features. It also views the ``object'' and ``attribute'' as playing
mutual roles, making their inferred representations invariant with respect to each other.

\section{Proposed Method}

\subsection{Preliminary}
\textbf{Problem formulation}
We start with a formal description of the domain generalizable person ReID. Suppose we have $K$ source domains (datasets) $\mathcal{S}=\{\mathcal{S}_{k}\}_{k=1}^{K}$. Each source domain contains its image-label pairs $\mathcal{S}_{k}=\left\{\left(\boldsymbol{x}_{i}^{k}, y_{i}^{k}\right)\right\}_{i=1}^{N_{k}}$. Each sample $\boldsymbol{x}_{i}^{k}$ is associated with an identity label $y_{i}^{k} \in \mathcal{Y}_{k}=\left\{1,2, \ldots, M_{k}\right\}$, where $M_{k}$ is the number of identities in the source dataset $\mathcal{S}_{k}$. Different domains have completely disjoint label spaces in the DG Re-ID setting; thus, the total number of identities from all source domains can be expressed as $M=\sum_{k=1}^{K} M_{k}$. In the training phase, we aggregate the samples from all source domains together with $\mathcal{S}=\{(\boldsymbol{x}_{i},y_{i},d_{i})\}_{i=1}^{N}$, where $N=\sum_{k=1}^K N_{k}$ denotes the number of samples from all the source domains, $y_{i} \in \{1,2,...,M\}$ and $d_{i} \in \{1,2,...,K\}$. During the test phase, we perform a person retrieval task on unseen target domains without additional model updating.

\textbf{Domain adversarial learning}
Domain adversarial learning is a promising method for achieving a generalizable model \cite{ganin2016domain,li2018domain}. We now briefly overview DAL and take it as a baseline model for generalizable person re-identification. In section 3.2, we will theoretically analyze the problem when directly applying DAL to person ReID and then propose our solution.

DAL aims to learn a domain-invariant feature representation that is discriminative for underlying tasks.
As shown in Figure 2(a), the entire feature extractor $G$ consists of $E$ and $M$. $E$ encodes the input image into a feature map, while $M$ maps the feature map into latent representation.
DAL uses a domain discriminator $D$ to predict domain labels from representation $z$ and simultaneously trains the feature extractor $G$ to remove domain information by deceiving $D$. Moreover, a classifier $C$ is utilized to learn discriminative latent representation. Formally, we use $f_{G}(x)$, $q_{D}(d|z)$ and $q_{C}(y|z)$ to denote the feature extractor, the discriminator and the classifier, respectively. Then, the loss function of DAL is expressed as follows:

\begin{equation}
	\begin{aligned}
		\min _{G, C} \max _{D} J(G, C, D)&=\mathbb{E}_{p(x, d, y)}\left[-L_{ADV}+L_{CE}\right], \\
		\quad \text { where } L_{ADV}&=-\log q_{D}\left(d \mid z=f_{G}(x)\right),\\
		\quad L_{CE}&=-\log q_{C}\left(y \mid z=f_{G}(x)\right) \label{e1}
	\end{aligned}
\end{equation}

Here, $L_{CE}$ is the cross-entropy loss that guarantees that the learned representation is sufficiently discriminative to classify people. $L_{ADV}$ is the min-max game between $G$ and $D$ that pushes the learned representation to be domain-invariant.

In person ReID, we directly use latent representation $z$ for retrieval during an inference phase. We mainly rely on the latent representation $z$ rather than the classifier $C$. Thus, to make equation (\ref{e1}) more explicit, for a specific $G$, we can assume that $C$ and $M$ obtain their optimal solutions $C^{*}$ and $M^{*}$ accordingly. The optimal $C^{*}$ and $M^{*}$ should reflect the true probability relation of the latent representation $z$. Formally, we have $q_{D^{*}}\left(d \mid z=f_{G}(x)\right)=p_{G}(d\mid z)$ and $q_{C^{*}}\left(y \mid z=f_{G}(x)\right)=p_{G}(y\mid z)$. Substituting $q_{D^{*}}$ and $q_{C^{*}}$ into equation (\ref{e1}), we can obtain the following concise and obvious loss function:

\begin{equation}
	\begin{aligned}
		\min _{G} J(G)&=- H_{p_{G}}(d \mid z)+H_{p_{G}} (y \mid z) \\
		&=-L_{ADV}+L_{CE}  \label{e2}
	\end{aligned}
\end{equation}

Here, $H_{p_{G}}(d \mid z)$ denotes the conditional entropy with the joint probability distribution $p_{G}(d, z)$. The first term attempts to learn $z$ that is invariant to $d$. Minimization of the second term requires the representation $z$ extracted by $G$ to contain as much information about $y$ as possible.

\subsection{Analysis of Domain-Class Correlation}
As mentioned in section 3.1, person ReID is a heterogeneous task where the identity space across different domains has no overlap. This phenomenon is due to thte deviations in the  collection of ReID data. Specifically, current ReID datasets are collected in specific locations at specific times; thus, different datasets have different identities. We show the casual graph of ReID in Figure \ref{fig1}(b). $h$ denotes deviations in the collection of ReID data and is the cause of both $y$ and $d$. $y$ and $d$ become correlated through a back door path $y \leftarrow z \rightarrow d $ \cite{pearl2016causal}. The mutual information satisfies $I(y,d) = H(d)-H(d|y)=H(y)-H(y|d)>0$. When we study the effect of $d$ on representation $z$, it will be disturbed by $y$ through the back-door path. Intuitively, the domain discriminator may predict domains according to class information. Then, through adversarial training, we may remove the unexpected class information in the latent representation.

Based on the above analysis, let us examine equation (\ref{e2}). Maximization of the first term $H_{p_{G}}(d \mid z)$ aims to learn domain-invariant representations. The optimal solution is $H_{p_{G}}(d \mid z)= H_{p_{G}}(d )$. The optimal solution of the second term is $H_{p_{G}} (y \mid z)=0 $, which means  that latent representation $z$ contains all information about $y$. However, in the following theorem, we will reveal that the optimal solution of the two terms cannot be satisfied simultaneously.

\begin{theorem}
	In ReID, when $H_{p_{G}}(d \mid z)=H(d)$ holds, $H_{p_{G}}(y \mid z) \geq H(d)$, i.e., the feature $z$ cannot contain all information about $y$. \label{t1}
\end{theorem}

\begin{proof}
	According to the definition of conditional entropy:
	\begin{equation}
		H(d \mid y) = -\sum\limits_{y \in \mathcal {Y},d \in \mathcal{D}} p(d,y)log \frac{p(d,y)}{p(y)} =0 \label{e14}
	\end{equation}
	Equation (\ref{e14}) is because for $y \in \mathcal {Y}_{d}$ we have $p(y)=p(d,y)$ and for $y  \notin \mathcal {Y}_{d}$ we have $p(d,y)=0$. Thus, the summation equals to zero. (Note: $0log0$ should be treated as being equal to zero.)
	
	Then, we use the property of conditional entropy to expand $H_{p_{G}}(d \mid h)$:
	\begin{equation}
		\begin{aligned}
			H_{p_{G}}(d \mid z) &\leq H_{p_{G}}(d, y \mid z)\\
			&=H_{p_{G}}(d \mid z, y)+H_{p_{G}}(y \mid z) \\
			& \leq H_{p_{G}}(d \mid y)+H_{p_{G}}(y \mid z) \label{e15}
		\end{aligned}
	\end{equation}
	
	\begin{equation} H_{p_{G}}(y \mid z)\geq
		H_{p_{G}}(d \mid z) - H_{p_{G}}(d \mid y) \label{e16}
	\end{equation}
	
	Substituting $H_{p_{G}}(d \mid z)=H(d)$ and $H_{p_{G}}(d \mid y)=0$ into equation (\ref{e16}), we obtain the following:
	\begin{equation}H_{p_{G}}(y \mid h)  \geq H(d)\end{equation}
\end{proof}
This theorem shows that if the representation is domain-invariant (the first term in Equation (\ref{e2}) reaches the optimum), the representation will lose some information about class, and the value of the information loss equals $H(d)$. This because $d$ and $y$ have information overlap; thus, when we remove information about $d$ in $z$, we lose information about $y$.

\subsection{Domain-Class Correlation Decomposition}
From the perspective of casual inference, we can decompose the domain-class correlation using intervention (``do'' operation) \cite{pearl2016causal,yao2020survey}. If we perform the operation on one variable, we fix it to a specific value and delete all edges pointing to the variable. The value of other variables may change accordingly. As shown in Figure \ref{fig1}(c), we perform intervention on $d$; thus, the back door path $d\leftarrow h \rightarrow y $ is cut off.

Now, the challenge is how to obtain the latent representation $z^{*}=z(x^{*})=z(x(y,do(d)))$ after we perform intervention on $d$. The back door adjustment criterion indicates that we can compute the conditional expectation based on the confounder $h$: $p(z \mid do(d)) = \sum \limits_{h}p(z \mid d,h)p(h)$. However, as mentioned above, $h$ is unobservable in ReID.

To tackle this problem, we propose to approximately estimate $z^{*}$ caused by intervention on $d$. Inspired by style transfer \cite{huang2017arbitrary,li2017demystifying} that uses the mean or covariance matrix of the feature maps to define style, we transform the feature map to match the mean and covariance of a specific domain. As shown in Figure \ref{fig3}, if we intervene on $d$ (assigning $d$ to be another specific domain $d^{*}$), we first estimate the feature map $f^{*}$ whose style is transferred to domain $d^{*}$). Then, we compute $z^{*}$ through mapping network $M$. Finally, we use the newly generated sample tuple $\{z^{*},y,d^{*}\}$ to compute the objective function and update the model.

\textbf{Whitening} Let $f_{i} \in \mathbb{R}^{H \times W \times C} $ be the feature map output by encoder $E$. We first subtract its own mean:

\begin{equation}
	\overline{f}_{i}  = f_{i}  - \mu(f),
\end{equation}
where $\mu(f_{i} ) \in \mathbb{R}^{C}$ denotes the mean of each channel. Then, we compute the covariance matrix:

\begin{equation}
	var(f_{i} ) = \overline{f}_{i}  \overline{f}^{T},
\end{equation}
where $var(f_{i} ) \in \mathbb{R}^{C \times C}$ is the covariance matrix of $f_{i} $.
We further whiten $f_{i} $ by matrix transformation:

\begin{equation}
	\hat{f}_{i}  =  W_{i}  \overline{f}_{i},
\end{equation}
where matrix $W_{i}$ satisfies $W_{i}^{T}W_{i} = var(f_{i})$. We use the highly  efficient and well-supported Cholesky decomposition to compute $W_{i}$. $\hat{f}_{i}$ is a whitened feature map that satisfies $\hat{f}_{i}\hat{f}_{i}^{T}=I$, where $I$ is the identity matrix.

\textbf{Memory bank}
We build two memory banks to store the covariance and mean matrices of each domain that represent the integral style of each domain. The memory banks are referred to as $V \in \mathbb{R}^{K \times C \times C} $and $U \in \mathbb{R}^{K \times C}$, where $K$ is the number of source domains. We use $V_{j}$ and $U_{j}$ and denote the covariance and mean matrices of domain $j$.

We update the memory with momentum during the training process. In the $t$-th iteration, the memory bank is updated according to:

\begin{equation}
	V_{j}^{t}=(1-\beta)V_{j}^{t-1}+\beta  \frac{\sum_{i=1}^{b}\mathbb{I}(d_{i}=j)var(f_{i})}{\sum_{i=1}^{b}\mathbb{I}(d_{i}=j)},
\end{equation}

\begin{equation}
	U_{j}^{t}=(1-\beta)U_{j}^{t-1}+\beta  \frac{\sum_{i=1}^{b}\mathbb{I}(d_{i}=j)\mu(f_{i})}{\sum_{i=1}^{b}\mathbb{I}(d_{i}=j)}.
\end{equation}
$\mathbb{I}$ is the characteristic function that satisfies:
$$ \mathbb{I}(d_{i}=j)=\left\{
\begin{aligned}
	1,&d_{i}=j\\
	0, &d_{i} \neq j
\end{aligned}
\right.
$$
$\beta$ is the momentum, and $b$ is the number of samples in the current iteration.

\begin{table*}[t]
	\centering
	\caption{Performance (\%) comparison with state-of-the-art methods. ``UDA'' denotes unsupervised domain adaptation and ``DG'' denotes domain generalization. ``M'' is MobileNet-V2, ``R'' is ResNet and ``A'' is AlexNet. Bold numbers denote the best performance.}\label{sota}

	\begin{tabular}{clccccccccccccccccccc}
		\midrule
		~	&\multirow{2}{*}{Method} &\multirow{2}{*}{Net} &\multicolumn{2}{c}{VIPeR}&\multicolumn{2}{c}{PRID}  &\multicolumn{2}{c}{GRID} &\multicolumn{2}{c}{i-LIDS}&\multicolumn{2}{c}{Average} \\	\cmidrule(rl){4-5}  \cmidrule(rl){6-7} \cmidrule(rl){8-9}  \cmidrule(rl){10-11} \cmidrule(rl){12-13}
		~	&	~& ~& R-1  & mAP & R-1  & mAP& R-1  & mAP& R-1  & mAP&R-1  & mAP\\
		
		\midrule

		\multirow{4}{*}{UDA}
		&TJAIDL \cite{wang2018transferable}&M& 38.5&-&34.8&-&-&-&-&-&-&-\\
		
		&MMFAN \cite{lin2018multi}&R & 39.1&-&35.1&-&-&-&-&-&-&-\\
		&CAMEL&R&30.9&&&&&&&&&\\
		&UDML \cite{peng2016unsupervised}&-&31.5&-&24.2&- &-&-&49.3&-&-&-\\		
		&Synthesis \cite{bak2018domain}&R&43.0&-&43.0&- &-&-&56.5&-&-&-\\		
		&SSDAL \cite{su2016deep} &A&37.9&-&20.1&-&19.1&-&-& \\
		\midrule
		\multirow{11}{*}{DG}	
		
		&DIMN \cite{song2019generalizable}&M   &51.2&60.1& 39.2 &52.0&29.3&41.1&70.2&78.4&47.5&57.9\\		
		
		&DualNorm \cite{jia2019frustratingly} &M &53.9&-&60.4&-&41.4&-&74.8&-&57.6&-\\

		&DDAN \cite{chen2020dual} &M &52.3&56.4& 54.5&58.9&50.6&55.7&78.5&81.5&59.0&63.1\\		
		
		&DDAN+DualNorm \cite{chen2020dual} &M &56.5&60.8& 62.9&67.5&46.2&50.9&78.0&81.2&60.9&65.1\\		
		
		&DCCDN (Ours)&M &53.1&57.4&57.9&62.4&51.6&56.4&81.5&84.2&61.0&65.1\\
		
		&DCCDN+DualNorm (Ours)&M &55.0&59.8&65.4&69.5&47.7&51.7&79.3&82.3&61.9&65.8\\
		
		\cmidrule(rl){2-13}

		&DualNorm \cite{jia2019frustratingly}  &R &59.4&-&69.6&-&43.7&-&78.2&-&62.7&-\\

		&DDAN \cite{chen2020dual} &R  &52.9&57.2& 57.5&62.4&50.5&55.1&78.7&81.7&59.9&64.1\\		
		
		&DDAN+DualNorm \cite{chen2020dual} &R &55.5&60.0& 72.2&76.0&50.2&54.6&81.0&84.0&64.7&68.7\\

		&DCCDN (Ours)&R &59.9&63.9&67.8&71.3&\textbf{60.3}&\textbf{64.0}&83.0&85.5&67.8&70.9\\
		
		&DCCDN+DualNorm (Ours)&R &\textbf{61.7}&\textbf{66.0}&\textbf{74.0}&\textbf{77.1}&54.2&59.1&\textbf{84.3}&\textbf{86.6}&\textbf{68.6}&\textbf{72.2}\\
		
		\midrule
	\end{tabular}
\end{table*}

\begin{table}[t]
	\centering
	\small
	\caption{Ablation studies of the proposed methods (
		the average performance on the DG ReID benchmark is reported). Bold numbers denote the best performance.
		\vspace{-3.5mm}
	}\label{Tab:table2}
	\begin{tabular}{lccccc}
		\midrule
		
		Method  & R-1 &R-5 &R-10&mAP\\	
		\midrule
		
		$L_{CE}$  & 59.4 &69.1 &79.9&63.8\\	
		
		$L_{CE}$ +$L_{Adv}$  &59.4&69.1&80.1&63.7\\
		
		$L_{CE}$+$L_{CE}^{*}$  &61.1&70.4 & 80.6&65.3\\
		
		$L_{CE}$+$L_{CE}^{*}$ +$L_{Adv}^{*}$ &62.7&71.3&81.6&66.8\\

		$L_{CE}$+$L_{CE}^{*}$+$Do$-$test$ &65.1&73.8&83.2&68.8\\
		
		$L_{CE}$+$L_{CE}^{*}$ +$L_{Adv}^{*}$+$Do$-$test$&\textbf{67.8}&\textbf{75.3}&\textbf{84.1}&\textbf{70.9}\\

		\midrule
	\end{tabular}
	\vspace{-3.5mm}
\end{table}

\textbf{Coloring}
Based on the memory bank that stores the style information of each domain, we can estimate the feature map $f^{*}$ whose style is transferred to domain $d^{*}$.
For example, if we $do(d_{i}=j)$, we take the covariance and mean matrices of domain $j$ from the memory bank. Then, we transform $f_{i}$ according to $V_{j}$ and $U_{j}$:

\begin{equation}
	f_{i}^{*} = W_{j}^{-1}\hat{f_{i}}+U_{j},
\end{equation}
where  $W_{j}^{T}W_{j} = V_{j}$. $f_{i}^{*}$ can be viewed as the feature of identity $y_{i}$ in domain $j$ that satisfies $\overline{f_{i}^{*}}\overline{f_{i}^{*}}^{T}=V_{j}$.

\textbf{Objective function}
We utilize newly generated tuples $\{ z^{*}, y, d^{*} \}$ to compute the objective function and update the model.

\begin{equation}
	\begin{aligned}
		\min _{G} J(G)&=\gamma \left[-  H_{p_{G}}(d^{*} \mid z^{*})+H_{p_{G}} (y \mid z^{*})  \right] + H_{p_{G}} (y \mid z) \\
		&= \gamma \left[ -L_{ADV}^{*}+L_{CE}^{*} \right] + L_{CE}\label{e9},
	\end{aligned}
\end{equation}

where $G $ is the composition of networks $E$ and $M$. We use DAL for the generated tuples whose domain-class correlation is decomposed (the first and second items in Equation (\ref{e9}). We also compute a cross-entropy using the original sample to ensure that we capture enough discriminative features about people.

\textbf{Test phase}
During the test phase, the target domain is unseen. We propose to intervene in the target domain data by changing its domain style using the statistical information of the source domains in the memory bank. Thus, the target domain data are closer to the source domain data from the style perspective. We compute the expectation of the representation according to the intervention as the final representation of the target domain data.

\begin{equation}
	\begin{aligned}
		\mathbb{E}_{do(d)}\left[ z^{*} \right] &= \mathbb{E}_{do(d)}\left[ z(y,do(d=i) \right] \\
		&= \sum_{i=1}^{K} p(do(d=i))  z(y,do(d=i))) \label{e10}
	\end{aligned}
\end{equation}

Here, $K$ is the number of source domains. We choose an even distribution for $p(do(d=i))$ for the sake of simplicity. In future work, we can examine the use of other distributions, such as the distribution of the true source domain.

\subsection{Theoretical Analysis}
In this section, we analyze the effect of domain adversarial learning. We first introduce the average total effect (ATE) \cite{pearl2016causal} that measures the effect of one factor on another factor. Here, we focus on the total effect of domain $d$ on the final latent representation $z$:

\begin{equation}
	ATE_{d \rightarrow z}= \mathbb{E}_{i \neq j}[z(do(d=i)-z(do(d=j))].
\end{equation}
It is the expectation of the residual of $z$ caused by different interventions on $d$.

Then, we review the adversarial loss function:

\begin{equation}
	\min _{G} \max _{D} V(G,D)= - \mathbb{E}_{p_{G}(z,d)} log(p(d \mid z ))
	\label{e12}
\end{equation}
Here, $p_{G}(z,d)$ is the joint distribution of $z$ and $d$ for a given network $G$. Let us denote the marginal distribution $z$ for each domain as $\{p_{G}^{1}(z), p_{G}^{2}(z), ..., p_{G}^{K}(z)\}$; then, Equation (\ref{e12}) can be written as:

\begin{equation}
	\min _{G} \max _{D} V(G,D)= - \sum_{i=1}^{K}\mathbb{E}_{p_{G}^{i}(z)} log(p(d=i \mid z )) \label{e16}
\end{equation}

\begin{lemma}
	The global optimal solution for $V(G,D)$ is attained if and only if $p_{G}^{1}(z) = p_{G}^{2}(z)= ... =p_{G}^{K}(z)$. \label{l1} \cite{li2018deep}
\end{lemma}

\begin{theorem}
	The global optimal solution of our objection function Equation (\ref{e9}) satisfies $ATE_{d \rightarrow z} = 0$. \label{t2}
\end{theorem}
\begin{proof}
	We first show that the first term $-  H_{p_{G}}(d^{*} \mid z^{*})$ and the second term $H_{p_{G}} (y \mid z^{*})$ in Equation (\ref{e9}) can reach the global optimum at the same time.
	When the first term reaches its global optimum, i.e., $H_{p_{G}}(d^{*} \mid z^{*} )=H_{p_{G}}(d^{*})$, using Equation (\ref{e15}), we obtain
\begin{equation}
		H_{p_{G}}(y \mid z^{*})\geq
		H_{p_{G}}(d^{*} \mid z^{*}) - H_{p_{G}}(d^{*} \mid y)  \label{e18}
	\end{equation}
	We note that since $y$ and $d^{*}$ are decomposed, we have $ H_{p_{G}}(d^{*} \mid y) =H_{p_{G}}(d^{*})$. Substituting into Equation (\ref{e18}), we obtain:
	
	\begin{equation}
		H_{p_{G}}(y \mid z^{*})\geq
		H_{p_{G}}(d^{*} \mid z^{*}) - H_{p_{G}}(d^{*} \mid y) =0 \label{e19}
	\end{equation}	
	
	Here, we can see that the lower bound in Equation (\ref{e15}) changes to 0 due to the decomposition of domain-class correlation. Therefore, the first and second terms can reach the global optimum at the same time.
	
	Then, when the first term reaches the optimum, according to Lemma \ref{l1}, we have $p_{G}^{1}(z^{*}) = p_{G}^{2}(z^{*})= ... =p_{G}^{K}(z^{*})$. Here, $p_{G}^{i}(z^{*})$ is the marginal distribution of the generated representation that is equal $p_{G}\left[ z(do(d=i))  \right]$. Thus, we have:
	\begin{equation}
		ATE_{d \rightarrow z}= \mathbb{E}_{i \neq j}[z(do(d=i)-z(do(d=j))] =0
	\end{equation}

\end{proof}

This clearly shows that our proposed method can effectively reduce the effect of the domain factor in the final representation.

\section{Experiment}
\subsection{Datasets and Settings}
\textbf{Datasets:} To evaluate the generalization ability of our proposed method, we conduct experiments on the large-scale DG ReID benchmark in \cite{song2019generalizable}. Source datasets contain CUHK02 \cite{li2013locally}, CUHK03 \cite{6909421}, Market1501 \cite{2015Scalable}, DukeMTMC-ReID \cite{8237667} and CUHK-SYSU PersonSearch \cite{8803643}. All images in the source datasets are used for training regardless of the train/test splits, and the total number of identities is $M$ = 18,530 with $N$ = 121,765 images.
Target datasets include VIPeR \cite{2008Viewpoint}, PRID \cite{hirzer2011person}, GRID \cite{DBLP:journals/ijcv/LoyXG10} and QMUL i-LDS \cite{peng2016unsupervised}. The detailed dataset statistics are given in Tab.~\ref{Tab:data_stastic}.

\textbf{Implementation details:} We adopt both MobileNetV2 with a width multiplier of 1.0 and ResNet50 as the backbone network for the composition of encoder $E$ and mapping network $M$. The mapping network is the last convolution layer of MobileNetV2 or the last 2 convolution layers of ResNet50. The discriminator $D$ is a fully connected layer with batch normalization and ReLU. To update our
base model, we use the SGD optimizer with a momentum of 0.9 and a weight decay of $5 \times 10^{-4}$. For MobileNetV2, the initial learning rate is 0.1, and it is trained for 100 epochs. For ResNet50, the initial learning rate is 0.05, and it is trained for 70 epochs. We decay the learning rate by 0.1 for every 40 epochs. The momentum  is set to $\beta=0.3$. The $\gamma$ in Equation \ref{e9} is set to 0.25.
\textbf{Evaluation metrics:} We follow the standard evaluation protocol in ReID. Specifically,
we adopt the cumulative matching characteristics (CMC) at Rank-1 (R-1) and mean average precision (mAP) for performance evaluation.

\subsection{Comparison with state-of-the-art methods}
In Table \ref{sota}, we compare the domain-class correlation decomposition network (DCCDN) with other state-of-the-art methods. We divide these methods into two groups: unsupervised domain adaptation (UDA) and domain generalization (DG).

\begin{figure}[t]
	\centering
	\subfigure[Different values of $\beta$.]{
		\begin{minipage}[t]{0.5\linewidth}
			\centering
			\includegraphics[width=1.7in]{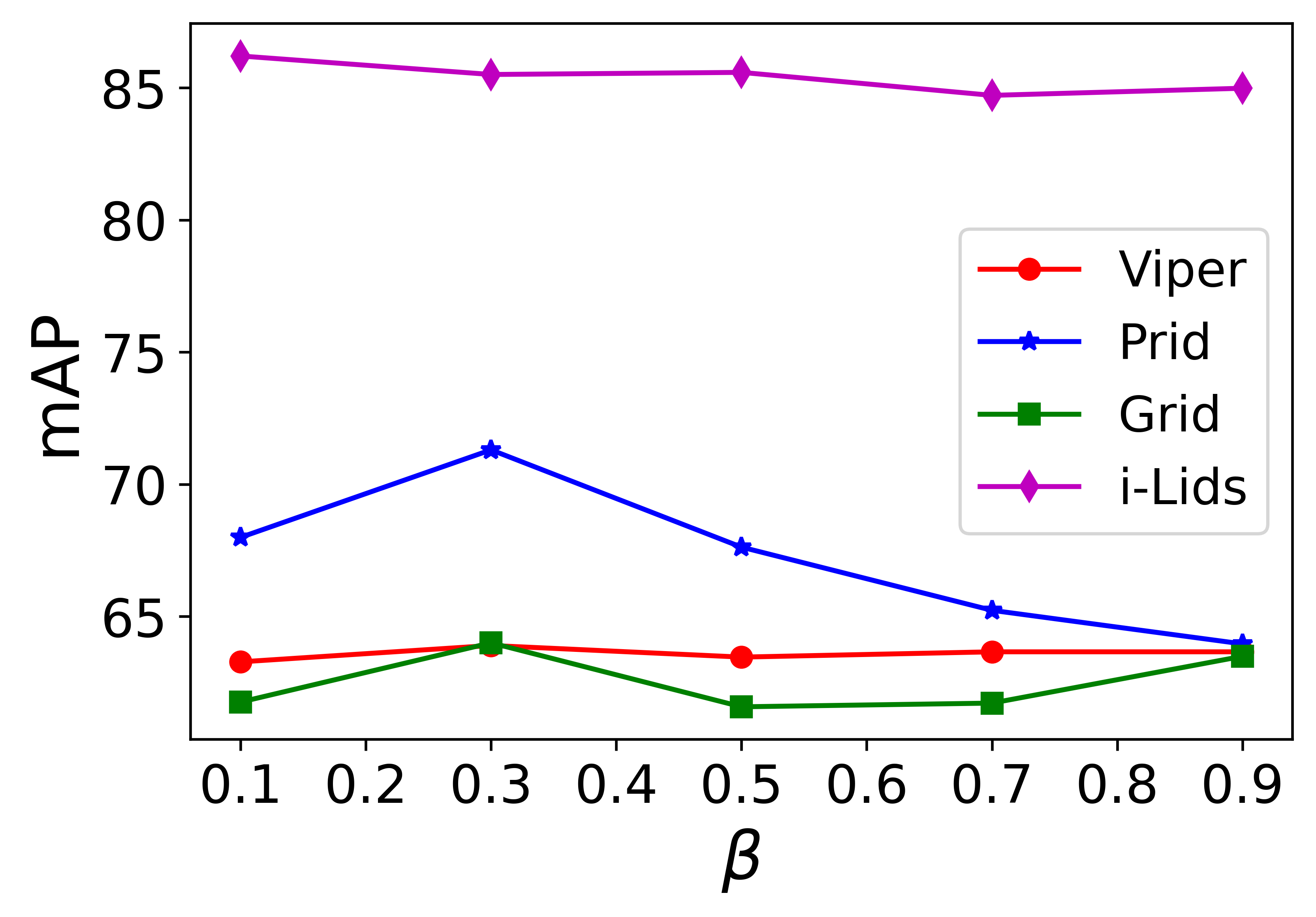}
		\end{minipage}%
	}%
	\subfigure[Different values of $\gamma$.]{
		\begin{minipage}[t]{0.5\linewidth}
			\centering
			\includegraphics[width=1.7in]{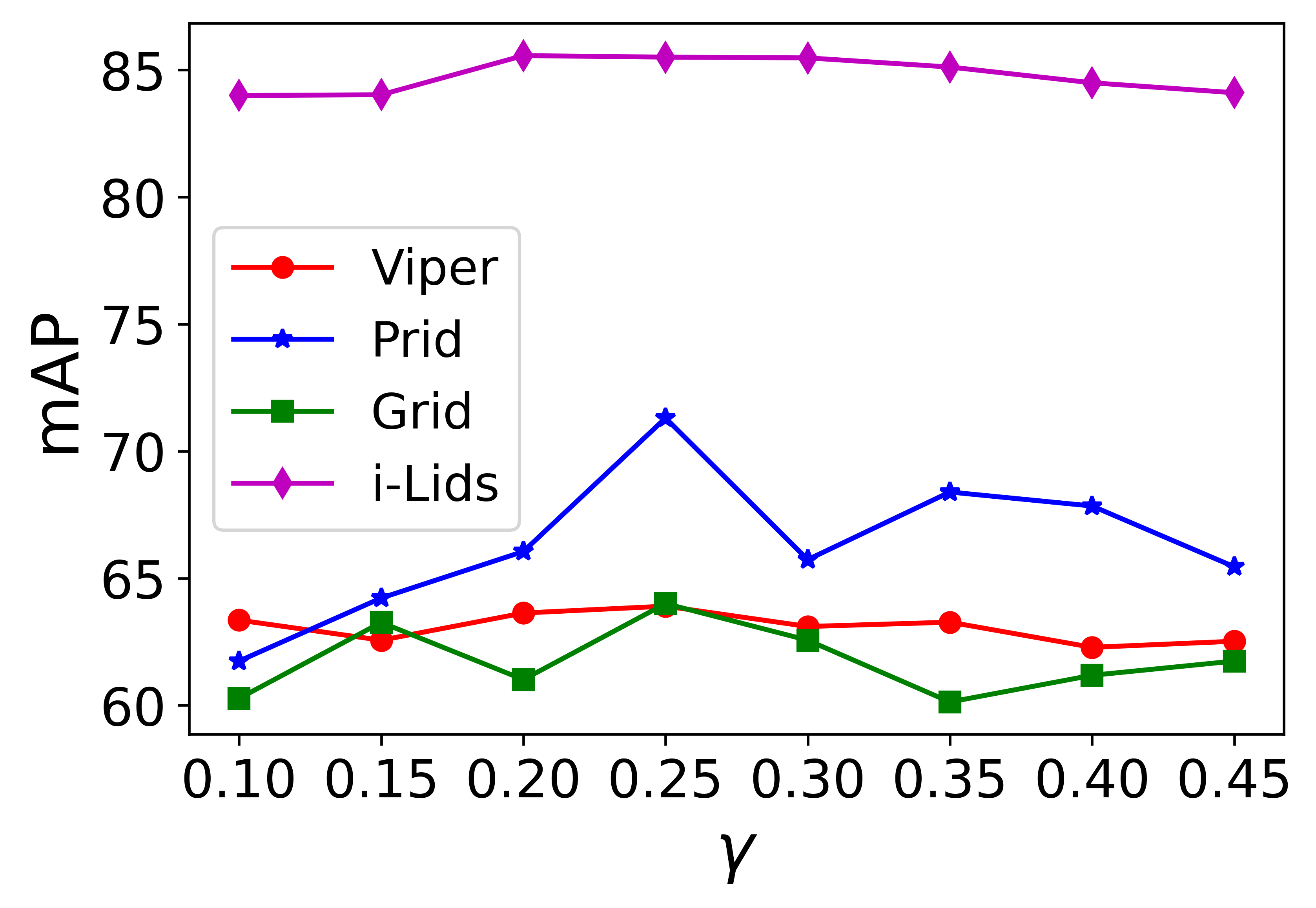}
		\end{minipage}%
	}%
	\centering
	\caption{ Evaluation with different values of import parameters.}\label{param}
\end{figure}

UDA methods focus on adapting a model trained on the source domain to the target domain. However, they differ from DG in that they ask for unlabeled data from the target domain. Nevertheless, our DCCDN significantly outperforms these UDA methods, indicating that the learned representations are generalizable.

Then, we compare with state-of-the-art DG methods for a fair comparison. We use both MobileNetV2 and ResNet50 as our network to evaluate the generality of DCCDN. Our DCCDN achieves the best mAP and R-1 on average for both MobileNetV2 and ResNet50. DIMN~\cite{song2019generalizable} relies on an additional mapping network to map the features into domain-invariant space, increasing the number of parameters and computational complexity. However, our DCCDN only adds a memory bank with few parameters and outperforms DIMN by a large margin on all of the datasets. Among these competitors, DualNorm \cite{jia2019frustratingly} that uses instance and batch normalization features and achieves comparable performance on GRID and i-LIDS. However, instance normalization may remove domain information and class information at the same time, which is the origin of the poor performance of DualNorm \cite{jia2019frustratingly} on GRID and i-LIDS. It is important to note that the normalization method in \cite{jia2019frustratingly} can be integrated into any network. Combined with Dualnorm \cite{jia2019frustratingly}, our DCCDN outperforms Dualnorm in all of the datasets. DDAN is another remarkable method \cite{chen2020dual}. Our DCCDN outperforms DDAN significantly, even without Dualnorm. This because  although DDAN \cite{chen2020dual} applies adversarial training, it will lose discriminative information due to domain-class correlation.
\subsection{Ablation study}
We conducted a comprehensive ablation study to demonstrate the effectiveness of each detailed component of our DCCDN. The results are summarized in Table \ref{Tab:table2}. The average performance of all target datasets is reported.

In Table \ref{Tab:table2}, $L_{CE}$ and $L_{ADV}$ denote cross-entropy loss and domain adversarial loss computed by original samples, respectively. $L_{CE}^{*}$ and $L_{ADV}^{*}$ are those computed by newly generated samples. $do$-$test$ denotes the use of Equation (\ref{e10}) to calculate the representation target domain data. Otherwise, we calculate the representation of the target domain data directly by forwarding the network. Comparing $L_{CE}$ and $L_{CE}$ +$L_{ADV}$, we find that $L_{ADV}$ has little effect because the original samples are domain-class correlated. However, comparing $L_{CE}$+$L_{CE}^{*}$+$Do$-$test$ with $L_{CE}$+$L_{CE}^{*}$ +$L_{Adv}^{*}$+$Do$-$test$, the performance gain provided by $L_{Adv}^{*}$ is clearly observed. This is because our newly generated samples are domain-class decomposed; thus, domain adversarial learning will not remove class information in the representations. Moreover, the proposed $do$-$test$ is highly effective because it transfers the domain style of the target domain data to that of the source domain data according to the memory bank.

\begin{table}[t]
		\centering
	\caption{Dataset stastics. ``Pr.": Probe, ``Ga.'':Gallery
	}\label{Tab:data_stastic}
	\begin{tabular}{cccccc}
		\toprule
		Domain                  & Dataset                  & \multicolumn{2}{c}{\# Train IDs} & \multicolumn{2}{c}{\# Train images} \\
		\midrule
		\multirow{5}{*}{Source} & CUHK02                   & \multicolumn{2}{c}{1816}         & \multicolumn{2}{c}{7264}            \\
		& CUHK03                   & \multicolumn{2}{c}{1467}         & \multicolumn{2}{c}{14097}           \\
		& Duke                     & \multicolumn{2}{c}{1812}         & \multicolumn{2}{c}{36411}           \\
		& Market1501               & \multicolumn{2}{c}{1501}         & \multicolumn{2}{c}{29419}           \\
		& PersonSearch             & \multicolumn{2}{c}{11934}        & \multicolumn{2}{c}{34574}           \\
		\midrule
		\multirow{2}{*}{Domain} & \multirow{2}{*}{Dataset} & \multicolumn{2}{c}{\# Test IDs}  & \multicolumn{2}{c}{\# Test images}  \\
		&                          & \#Pr. IDs       & \#Ga. IDs      & \#Pr. imgs       & \#Ga. imgs       \\
		\midrule
		\multirow{4}{*}{Target} & VIPeR                    & 316             & 316            & 316              & 316              \\
		& PRID                     & 100             & 649            & 100              & 649              \\
		& GRID                     & 125             & 900            & 125              & 1025             \\
		& i-LIDS                   & 60              & 60             & 60               & 60               \\
		\bottomrule
	\end{tabular}
\end{table}

\begin{table}[t]
	\centering
	\caption{Study on the layer for DCCD. Bold numbers denote the best performance.
	}\label{which_layer}
	\begin{tabular}{lccccc}
		\midrule
	
		Method  & R-1 &R-5 &R-10&mAP\\	
		\midrule
		$baseline$ &59.4&69.1&80.1&63.7\\	
		$layer1$  & 64.6 &73.1 &\textbf{84.4}&68.1\\			
		$layer2$&\textbf{67.8}&\textbf{75.3}&84.1&\textbf{70.9}\\
		$layer3$  &64.0&73.8&84.0&67.4\\				
		\midrule
	\end{tabular}
\end{table}

 \begin{table}[t]
	\centering
	\caption{Performance (\%) on large-scale datasets, where C2, C3, CS and D denote CUHK02, CUHK03, CUHKSYSU and DukeMTMC, respectively. }\label{large_dataset}
	\begin{tabular}{clcccc}
		\hline
		\multirow{2}{*}{Source} 	&\multirow{2}{*}{Method} &\multicolumn{2}{c}{Target: Market1501}&\multicolumn{2}{c}{Target: MSMT17}\\	
		\cline{3-6}
		~	&	~& R-1  & mAP & R-1  & mAP\\
		\hline
		\multirow{2}{*}{\makecell[c]{C2+C3\\+CS+D}}&Baseline &83.3&63.7&32.0&12.9 \\
		\cline{2-6}
		~&DCCDN&85.6&65.9&42.1&17.5\\
		\hline
	\end{tabular}
\end{table}

\begin{table}[t]
	\centering
	\caption{Performance (\%) with different amounts of source samples, where M, C2, C3, CS and D denotee Market1501, CUHK02, CUHK03, CUHKSYSU and DukeMTMC, respectively. }\label{less_data}
	\begin{tabular}{clcc}
		\hline
		\multirow{2}{*}{Source} 	&\multirow{2}{*}{Method} &\multicolumn{2}{c}{Target: Average}\\	
		\cline{3-4}
		~	&	~& R-1  & mAP \\
		\hline
		\multirow{2}{*}{\makecell[c]{M+C2+C3+CS+D \\ data num. 122k}}&Baseline &59.4&63.7 \\
		\cline{2-4}
		~&DCCDN&67.8&70.9\\	
		\hline
		\multirow{2}{*}{\makecell[c]{M+C2+C3+D \\ data num. 87k}}&Baseline &51.2&61.0 \\
		\cline{2-4}
		~&DCCDN&58.1&67.1\\		
		\hline
		\multirow{2}{*}{\makecell[c]{M+C3+D \\ data num. 79k}}&Baseline &49.7&59.5 \\
		\cline{2-4}
		~&DCCDN&57.4&66.2\\
		\hline
	\end{tabular}
\end{table}

\begin{figure*}[t]
	\centering
	\subfigure[$L_{CE}$ (ResNet-50).]{
		\begin{minipage}[t]{0.5\linewidth}
			\centering
			\includegraphics[width=3.8in]{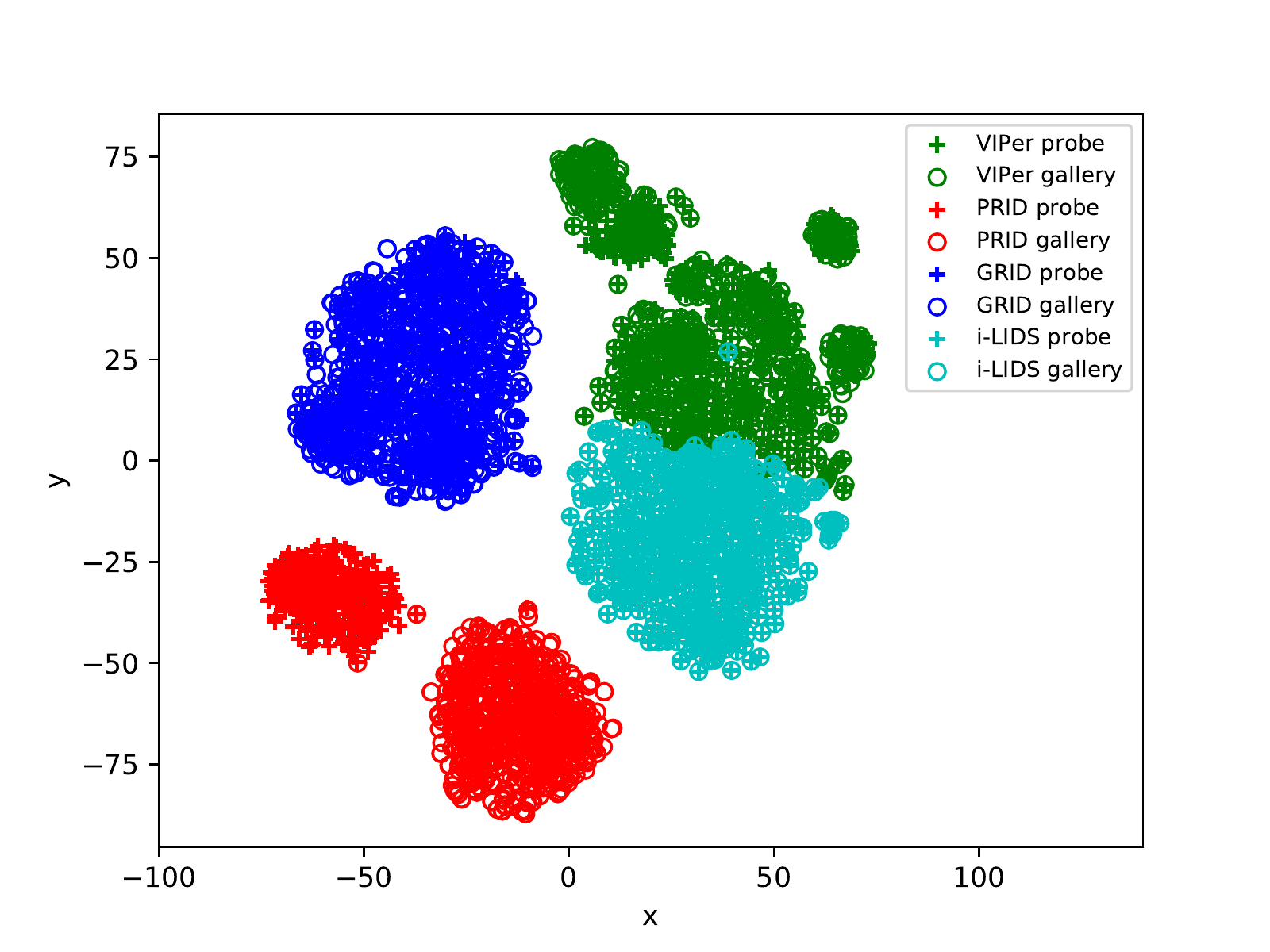}
		\end{minipage}%
	}%
	\subfigure[DCCDN (ResNet-50).]{
		\begin{minipage}[t]{0.5\linewidth}
			\centering
			\includegraphics[width=3.8in]{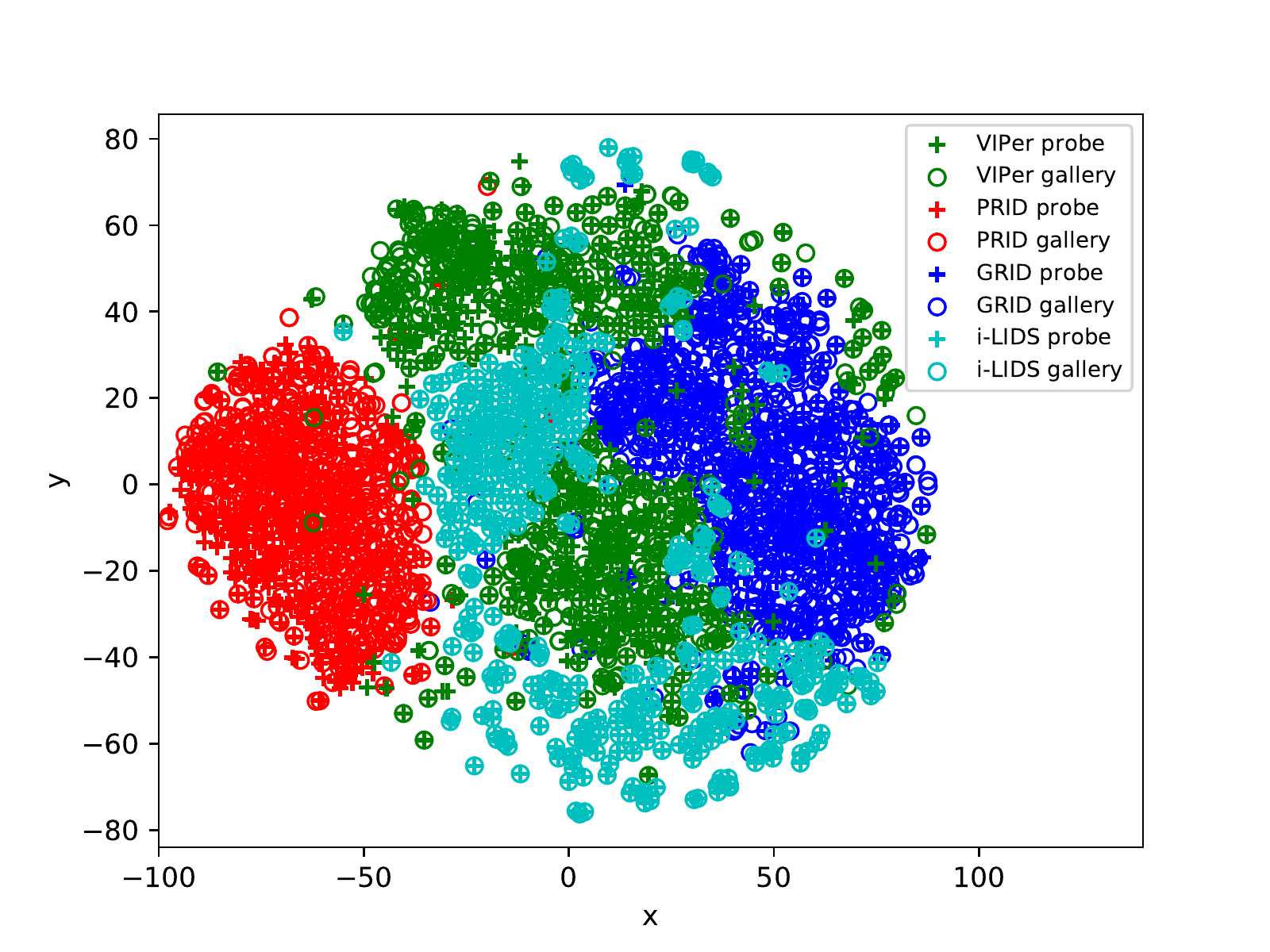}
		\end{minipage}%
	}%
	\centering
	\caption{The t-SNE visualization of the embedding vectors on four unseen target datasets (VIPeR, PRID, GRID, and i-LIDS).}	\label{tsne}
\end{figure*}

\begin{figure}[t]
	\centering
	\subfigure[CMC Curve of DCCDN.]{
		\begin{minipage}[t]{0.5\linewidth}
			\centering
			\includegraphics[width=1.75in]{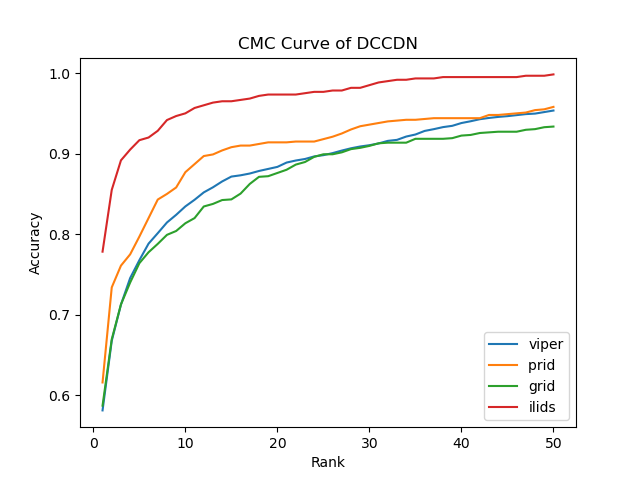}
		\end{minipage}%
	}%
	\subfigure[CMC Curve of Baseline.]{
		\begin{minipage}[t]{0.5\linewidth}
			\centering
			\includegraphics[width=1.75in]{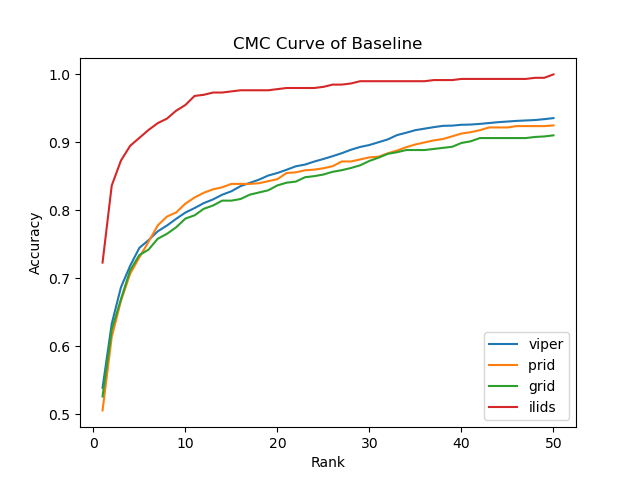}
		\end{minipage}%
	}%
	\centering
	\caption{ CMC Cureve.}	\label{cmc}
\end{figure}

\subsection{Study on the layer for performing DCCD} We conduct the proposed domain-class correlation decomposition in different layers of ResNet50 and report the results in Table \ref{which_layer}. It is clearly observed that in each layer, the DCCD module can provide significant performance gain. The best performance is obtained for insertion into layer 2.

\subsection{Important Parameters} We further investigate the impact of two important hyperparameters: the update momentum $\beta$ and the loss weight $\gamma$. The results are given in Fig.~\ref{param}. We find that the datasets i-Lids and Viper are robust with respect to parameter variations, while Prid and Grid are relatively sensitive. When we set the parameters $\beta=0.3$ and $\gamma=0.25$, the model achieves best performance. The curves are flat on VIPer and i-LIDS, showing that our model is robust with respect to parameters.

\subsection{T-SNE Visualization of Embeddings}
We use t-SNE~\cite{hinton2002stochastic} to visualize the embedding vectors of ResNet-50 on four unseen target datasets. We compare our DCCDN with training ResNet50 using only $L_{CE}$ loss and show the results in Figs.~\ref{tsne}. (a) and Fig.~\ref{tsne}. (b) respectively. Comparison of these two visualization results clearly shows that DCCDN successfully learns domain-invariant representations. The representations of DCCDN on four target datasets are clustered together, while those of $L_{CE}$ are separated. Moreover, the probe and gallery representations of $L_{CE}$ on VIPer are separated, which means that the representations are not discriminative. However, the probe and gallery representations of DCCDN on the same datasets are close to each other. This proves that DCCDN learns both domain-invariant and discriminative representations, benefiting from adversarial learning and domain-class correlation decomposition.

\subsection{Other Analysis}
\textbf{Performance on Large-scale Datasets.} In Table~\ref{large_dataset}, we show the performance of the baseline and DCCDN on two large-scale datasets: Market1501~\cite{zheng2015scalable} and MSMT17~\cite{wei2018person}. It is clearly observed  that DCCDN outperforms the baseline on both datasets, and particularly on MSMT17 that is a multiscene multitime person ReID dataset.

\textbf{Performance with less source data.} We show the results of baseline and DCCDN with different numbers of source domains in Table~\ref{less_data}. It is observed that the performance of DCCDN is not sensitive to source domain numbers. DCCDN outperforms the baseline by a large margin on all three settings.

\textbf{CMC Curves.} The CMC curves of the baseline and DCCDN on the four target datasets are given in Fig.~\ref{cmc}

\section{Conclusion}
In this work, we theoretically and experimentally confirm that in ReID, domain and class exhibit correlations, and these correlations will lead to the loss of discriminative features. To solve this problem, we have proposed a generalizable person ReID framework, called the domain-class correlation decomposition network. It decomposes such correlation and learns both generalizable and discriminative representations.
Extensive experiments and comprehensive analyses on the large-scale DG benchmark demonstrate its superiority over the state-of-the-art methods.


%

\ifCLASSOPTIONcaptionsoff
  \newpage
\fi



\bibliographystyle{IEEEtran}
\bibliography{IEEEabrv, paper}

\begin{thebibliography}{10}
\providecommand{\url}[1]{#1}
\csname url@samestyle\endcsname
\providecommand{\newblock}{\relax}
\providecommand{\bibinfo}[2]{#2}
\providecommand{\BIBentrySTDinterwordspacing}{\spaceskip=0pt\relax}
\providecommand{\BIBentryALTinterwordstretchfactor}{4}
\providecommand{\BIBentryALTinterwordspacing}{\spaceskip=\fontdimen2\font plus
\BIBentryALTinterwordstretchfactor\fontdimen3\font minus
  \fontdimen4\font\relax}
\providecommand{\BIBforeignlanguage}[2]{{%
\expandafter\ifx\csname l@#1\endcsname\relax
\typeout{** WARNING: IEEEtran.bst: No hyphenation pattern has been}%
\typeout{** loaded for the language `#1'. Using the pattern for}%
\typeout{** the default language instead.}%
\else
\language=\csname l@#1\endcsname
\fi
#2}}
\providecommand{\BIBdecl}{\relax}
\BIBdecl

\bibitem{he2016deep}
K.~He, X.~Zhang, S.~Ren, and J.~Sun, ``Deep residual learning for image
  recognition,'' in \emph{CVPR}, 2016, pp. 770--778.

\bibitem{yang2018local}
J.~Yang, X.~Shen, X.~Tian, H.~Li, J.~Huang, and X.-S. Hua, ``Local
  convolutional neural networks for person re-identification,'' in \emph{ACM
  MM}, 2018, pp. 1074--1082.

\bibitem{luo2019strong}
H.~Luo, W.~Jiang, Y.~Gu, F.~Liu, X.~Liao, S.~Lai, and J.~Gu, ``A strong
  baseline and batch normalization neck for deep person re-identification,''
  \emph{IEEE Transactions on Multimedia}, vol.~22, no.~10, pp. 2597--2609,
  2019.

\bibitem{ge2020mutual}
Y.~Ge, D.~Chen, and H.~Li, ``Mutual mean-teaching: Pseudo label refinery for
  unsupervised domain adaptation on person re-identification,'' \emph{arXiv
  preprint arXiv:2001.01526}, 2020.

\bibitem{li2018domain}
H.~Li, S.~Jialin~Pan, S.~Wang, and A.~C. Kot, ``Domain generalization with
  adversarial feature learning,'' in \emph{CVPR}, 2018, pp. 5400--5409.

\bibitem{ganin2016domain}
Y.~Ganin, E.~Ustinova, H.~Ajakan, P.~Germain, H.~Larochelle, F.~Laviolette,
  M.~Marchand, and V.~Lempitsky, ``Domain-adversarial training of neural
  networks,'' \emph{JMLR}, vol.~17, no.~1, pp. 2096--2030, 2016.

\bibitem{li2018deep}
Y.~Li, X.~Tian, M.~Gong, Y.~Liu, T.~Liu, K.~Zhang, and D.~Tao, ``Deep domain
  generalization via conditional invariant adversarial networks,'' in
  \emph{ECCV}, 2018, pp. 624--639.

\bibitem{akuzawa2019adversarial}
K.~Akuzawa, Y.~Iwasawa, and Y.~Matsuo, ``Adversarial invariant feature learning
  with accuracy constraint for domain generalization,'' in \emph{KDD}.\hskip
  1em plus 0.5em minus 0.4em\relax Springer, 2019, pp. 315--331.

\bibitem{pearl2016causal}
J.~Pearl, M.~Glymour, and N.~P. Jewell, \emph{Causal inference in statistics: A
  primer}.\hskip 1em plus 0.5em minus 0.4em\relax John Wiley \& Sons, 2016.

\bibitem{yao2020survey}
L.~Yao, Z.~Chu, S.~Li, Y.~Li, J.~Gao, and A.~Zhang, ``A survey on causal
  inference,'' \emph{arXiv preprint arXiv:2002.02770}, 2020.

\bibitem{huang2017arbitrary}
X.~Huang and S.~Belongie, ``Arbitrary style transfer in real-time with adaptive
  instance normalization,'' in \emph{ICCV}, 2017, pp. 1501--1510.

\bibitem{li2017demystifying}
Y.~Li, N.~Wang, J.~Liu, and X.~Hou, ``Demystifying neural style transfer,''
  \emph{arXiv preprint arXiv:1701.01036}, 2017.

\bibitem{zhou2020learning}
K.~Zhou, Y.~Yang, T.~Hospedales, and T.~Xiang, ``Learning to generate novel
  domains for domain generalization,'' in \emph{European Conference on Computer
  Vision}.\hskip 1em plus 0.5em minus 0.4em\relax Springer, 2020, pp. 561--578.

\bibitem{qiao2020learning}
F.~Qiao, L.~Zhao, and X.~Peng, ``Learning to learn single domain
  generalization,'' in \emph{Proceedings of the IEEE/CVF Conference on Computer
  Vision and Pattern Recognition}, 2020, pp. 12\,556--12\,565.

\bibitem{tolstikhin2017wasserstein}
I.~Tolstikhin, O.~Bousquet, S.~Gelly, and B.~Schoelkopf, ``Wasserstein
  auto-encoders,'' \emph{arXiv preprint arXiv:1711.01558}, 2017.

\bibitem{zhou2021domain}
K.~Zhou, Y.~Yang, Y.~Qiao, and T.~Xiang, ``Domain generalization with
  mixstyle,'' \emph{arXiv preprint arXiv:2104.02008}, 2021.

\bibitem{zhang2017mixup}
H.~Zhang, M.~Cisse, Y.~N. Dauphin, and D.~Lopez-Paz, ``mixup: Beyond empirical
  risk minimization,'' \emph{arXiv preprint arXiv:1710.09412}, 2017.

\bibitem{zhou2020domain}
F.~Zhou, Z.~Jiang, C.~Shui, B.~Wang, and B.~Chaib-draa, ``Domain generalization
  with optimal transport and metric learning,'' \emph{arXiv preprint
  arXiv:2007.10573}, 2020.

\bibitem{pan2010domain}
S.~J. Pan, I.~W. Tsang, J.~T. Kwok, and Q.~Yang, ``Domain adaptation via
  transfer component analysis,'' \emph{IEEE Transactions on Neural Networks},
  vol.~22, no.~2, pp. 199--210, 2010.

\bibitem{ghifary2016scatter}
M.~Ghifary, D.~Balduzzi, W.~B. Kleijn, and M.~Zhang, ``Scatter component
  analysis: A unified framework for domain adaptation and domain
  generalization,'' \emph{IEEE transactions on pattern analysis and machine
  intelligence}, vol.~39, no.~7, pp. 1414--1430, 2016.

\bibitem{li2018domain2}
Y.~Li, M.~Gong, X.~Tian, T.~Liu, and D.~Tao, ``Domain generalization via
  conditional invariant representations,'' in \emph{Proceedings of the AAAI
  Conference on Artificial Intelligence}, vol.~32, no.~1, 2018.

\bibitem{sun2019learning}
Y.~Sun, L.~Zheng, Y.~Li, Y.~Yang, Q.~Tian, and S.~Wang, ``Learning part-based
  convolutional features for person re-identification,'' \emph{IEEE
  transactions on pattern analysis and machine intelligence}, 2019.

\bibitem{wang2018learning}
G.~Wang, Y.~Yuan, X.~Chen, J.~Li, and X.~Zhou, ``Learning discriminative
  features with multiple granularities for person re-identification,'' in
  \emph{Proceedings of the 26th ACM international conference on Multimedia},
  2018, pp. 274--282.

\bibitem{wan2019concentrated}
C.~Wan, Y.~Wu, X.~Tian, J.~Huang, and X.-S. Hua, ``Concentrated local part
  discovery with fine-grained part representation for person
  re-identification,'' \emph{IEEE Transactions on Multimedia}, vol.~22, no.~6,
  pp. 1605--1618, 2019.

\bibitem{zhao2017spindle}
H.~Zhao, M.~Tian, S.~Sun, J.~Shao, J.~Yan, S.~Yi, X.~Wang, and X.~Tang,
  ``Spindle net: Person re-identification with human body region guided feature
  decomposition and fusion,'' in \emph{Proceedings of the IEEE conference on
  computer vision and pattern recognition}, 2017, pp. 1077--1085.

\bibitem{wei2017glad}
L.~Wei, S.~Zhang, H.~Yao, W.~Gao, and Q.~Tian, ``Glad: Global-local-alignment
  descriptor for pedestrian retrieval,'' in \emph{Proceedings of the 25th ACM
  international conference on Multimedia}, 2017, pp. 420--428.

\bibitem{sarfraz2018pose}
M.~S. Sarfraz, A.~Schumann, A.~Eberle, and R.~Stiefelhagen, ``A pose-sensitive
  embedding for person re-identification with expanded cross neighborhood
  re-ranking,'' in \emph{Proceedings of the IEEE Conference on Computer Vision
  and Pattern Recognition}, 2018, pp. 420--429.

\bibitem{si2018dual}
J.~Si, H.~Zhang, C.-G. Li, J.~Kuen, X.~Kong, A.~C. Kot, and G.~Wang, ``Dual
  attention matching network for context-aware feature sequence based person
  re-identification,'' in \emph{Proceedings of the IEEE Conference on Computer
  Vision and Pattern Recognition}, 2018, pp. 5363--5372.

\bibitem{li2018harmonious}
W.~Li, X.~Zhu, and S.~Gong, ``Harmonious attention network for person
  re-identification,'' in \emph{Proceedings of the IEEE conference on computer
  vision and pattern recognition}, 2018, pp. 2285--2294.

\bibitem{li2018diversity}
S.~Li, S.~Bak, P.~Carr, and X.~Wang, ``Diversity regularized spatiotemporal
  attention for video-based person re-identification,'' in \emph{Proceedings of
  the IEEE Conference on Computer Vision and Pattern Recognition}, 2018, pp.
  369--378.

\bibitem{zou2020joint}
Y.~Zou, X.~Yang, Z.~Yu, B.~Kumar, and J.~Kautz, ``Joint disentangling and
  adaptation for cross-domain person re-identification,'' \emph{arXiv preprint
  arXiv:2007.10315}, 2020.

\bibitem{zhong2018generalizing}
Z.~Zhong, L.~Zheng, S.~Li, and Y.~Yang, ``Generalizing a person retrieval model
  hetero-and homogeneously,'' in \emph{Proceedings of the European Conference
  on Computer Vision (ECCV)}, 2018, pp. 172--188.

\bibitem{chen2019instance}
Y.~Chen, X.~Zhu, and S.~Gong, ``Instance-guided context rendering for
  cross-domain person re-identification,'' in \emph{Proceedings of the IEEE/CVF
  International Conference on Computer Vision}, 2019, pp. 232--242.

\bibitem{deng2018image}
W.~Deng, L.~Zheng, Q.~Ye, G.~Kang, Y.~Yang, and J.~Jiao, ``Image-image domain
  adaptation with preserved self-similarity and domain-dissimilarity for person
  re-identification,'' in \emph{Proceedings of the IEEE conference on computer
  vision and pattern recognition}, 2018, pp. 994--1003.

\bibitem{li2021intra}
Y.~Li, H.~Yao, and C.~Xu, ``Intra-domain consistency enhancement for
  unsupervised person re-identification,'' \emph{IEEE Transactions on
  Multimedia}, 2021.

\bibitem{xie2020progressive}
Q.~Xie, W.~Zhou, G.-J. Qi, Q.~Tian, and H.~Li, ``Progressive unsupervised
  person re-identification by tracklet association with spatio-temporal
  regularization,'' \emph{IEEE Transactions on Multimedia}, 2020.

\bibitem{ding2020adaptive}
Y.~Ding, H.~Fan, M.~Xu, and Y.~Yang, ``Adaptive exploration for unsupervised
  person re-identification,'' \emph{ACM Transactions on Multimedia Computing,
  Communications, and Applications (TOMM)}, vol.~16, no.~1, pp. 1--19, 2020.

\bibitem{fan2018unsupervised}
H.~Fan, L.~Zheng, C.~Yan, and Y.~Yang, ``Unsupervised person re-identification:
  Clustering and fine-tuning,'' \emph{ACM Transactions on Multimedia Computing,
  Communications, and Applications (TOMM)}, vol.~14, no.~4, pp. 1--18, 2018.

\bibitem{song2020unsupervised}
L.~Song, C.~Wang, L.~Zhang, B.~Du, Q.~Zhang, C.~Huang, and X.~Wang,
  ``Unsupervised domain adaptive re-identification: Theory and practice,''
  \emph{Pattern Recognition}, vol. 102, p. 107173, 2020.

\bibitem{tang2019unsupervised}
H.~Tang, Y.~Zhao, and H.~Lu, ``Unsupervised person re-identification with
  iterative self-supervised domain adaptation,'' in \emph{Proceedings of the
  IEEE/CVF Conference on Computer Vision and Pattern Recognition Workshops},
  2019, pp. 0--0.

\bibitem{yu2019unsupervised}
H.-X. Yu, W.-S. Zheng, A.~Wu, X.~Guo, S.~Gong, and J.-H. Lai, ``Unsupervised
  person re-identification by soft multilabel learning,'' in \emph{Proceedings
  of the IEEE/CVF Conference on Computer Vision and Pattern Recognition}, 2019,
  pp. 2148--2157.

\bibitem{fu2019self}
Y.~Fu, Y.~Wei, G.~Wang, Y.~Zhou, H.~Shi, and T.~S. Huang, ``Self-similarity
  grouping: A simple unsupervised cross domain adaptation approach for person
  re-identification,'' in \emph{Proceedings of the IEEE/CVF International
  Conference on Computer Vision}, 2019, pp. 6112--6121.

\bibitem{song2019generalizable}
J.~Song, Y.~Yang, Y.-Z. Song, T.~Xiang, and T.~M. Hospedales, ``Generalizable
  person re-identification by domain-invariant mapping network,'' in
  \emph{CVPR}, 2019, pp. 719--728.

\bibitem{jia2019frustratingly}
J.~Jia, Q.~Ruan, and T.~M. Hospedales, ``Frustratingly easy person
  re-identification: Generalizing person re-id in practice,'' \emph{arXiv
  preprint arXiv:1905.03422}, 2019.

\bibitem{ulyanov2016instance}
D.~Ulyanov, A.~Vedaldi, and V.~Lempitsky, ``Instance normalization: The missing
  ingredient for fast stylization,'' \emph{arXiv preprint arXiv:1607.08022},
  2016.

\bibitem{chen2020dual}
P.~Chen, P.~Dai, J.~Liu, F.~Zheng, Q.~Tian, and R.~Ji, ``Dual distribution
  alignment network for generalizable person re-identification,'' \emph{arXiv
  preprint arXiv:2007.13249}, 2020.

\bibitem{gatys2016image}
L.~A. Gatys, A.~S. Ecker, and M.~Bethge, ``Image style transfer using
  convolutional neural networks,'' in \emph{Proceedings of the IEEE conference
  on computer vision and pattern recognition}, 2016, pp. 2414--2423.

\bibitem{li2017universal}
Y.~Li, C.~Fang, J.~Yang, Z.~Wang, X.~Lu, and M.-H. Yang, ``Universal style
  transfer via feature transforms,'' in \emph{NeurIPS}, 2017, pp. 386--396.

\bibitem{li2019learning}
X.~Li, S.~Liu, J.~Kautz, and M.-H. Yang, ``Learning linear transformations for
  fast image and video style transfer,'' in \emph{Proceedings of the IEEE/CVF
  Conference on Computer Vision and Pattern Recognition}, 2019, pp. 3809--3817.

\bibitem{gong2016domain}
M.~Gong, K.~Zhang, T.~Liu, D.~Tao, C.~Glymour, and B.~Sch{\"o}lkopf, ``Domain
  adaptation with conditional transferable components,'' in \emph{ICML}, 2016,
  pp. 2839--2848.

\bibitem{heinze2020conditional}
C.~Heinze-Deml and N.~Meinshausen, ``Conditional variance penalties and domain
  shift robustness,'' \emph{Machine Learning}, pp. 1--46, 2020.

\bibitem{wang2018transferable}
J.~Wang, X.~Zhu, S.~Gong, and W.~Li, ``Transferable joint attribute-identity
  deep learning for unsupervised person re-identification,'' in \emph{CVPR},
  2018, pp. 2275--2284.

\bibitem{lin2018multi}
S.~Lin, H.~Li, C.-T. Li, and A.~C. Kot, ``Multi-task mid-level feature
  alignment network for unsupervised cross-dataset person re-identification,''
  \emph{arXiv preprint arXiv:1807.01440}, 2018.

\bibitem{peng2016unsupervised}
P.~Peng, T.~Xiang, Y.~Wang, M.~Pontil, S.~Gong, T.~Huang, and Y.~Tian,
  ``Unsupervised cross-dataset transfer learning for person
  re-identification,'' in \emph{CVPR}, 2016, pp. 1306--1315.

\bibitem{bak2018domain}
S.~Bak, P.~Carr, and J.-F. Lalonde, ``Domain adaptation through synthesis for
  unsupervised person re-identification,'' in \emph{ECCV}, 2018, pp. 189--205.

\bibitem{su2016deep}
C.~Su, S.~Zhang, J.~Xing, W.~Gao, and Q.~Tian, ``Deep attributes driven
  multi-camera person re-identification,'' in \emph{European conference on
  computer vision}.\hskip 1em plus 0.5em minus 0.4em\relax Springer, 2016, pp.
  475--491.

\bibitem{li2013locally}
W.~Li and X.~Wang, ``Locally aligned feature transforms across views,'' in
  \emph{Proceedings of the IEEE conference on computer vision and pattern
  recognition}, 2013, pp. 3594--3601.

\bibitem{6909421}
W.~{Li}, R.~{Zhao}, T.~{Xiao}, and X.~{Wang}, ``Deepreid: Deep filter pairing
  neural network for person re-identification,'' in \emph{CVPR}, 2014, pp.
  152--159.

\bibitem{2015Scalable}
L.~Zheng, L.~Shen, L.~Tian, S.~Wang, and Q.~Tian, ``Scalable person
  re-identification: A benchmark,'' in \emph{ICCV}, 2015.

\bibitem{8237667}
Z.~{Zheng}, L.~{Zheng}, and Y.~{Yang}, ``Unlabeled samples generated by gan
  improve the person re-identification baseline in vitro,'' in \emph{ICCV},
  2017, pp. 3774--3782.

\bibitem{8803643}
A.~{Loesch}, J.~{Rabarisoa}, and R.~{Audigier}, ``End-to-end person search
  sequentially trained on aggregated dataset,'' in \emph{ICIP, 2019}, 2019, pp.
  4574--4578.

\bibitem{2008Viewpoint}
D.~Gray and H.~Tao, ``Viewpoint invariant pedestrian recognition with an
  ensemble of localized features,'' 2008.

\bibitem{hirzer2011person}
M.~Hirzer, C.~Beleznai, P.~M. Roth, and H.~Bischof, ``Person re-identification
  by descriptive and discriminative classification,'' in \emph{SCIA}.\hskip 1em
  plus 0.5em minus 0.4em\relax Springer, 2011, pp. 91--102.

\bibitem{DBLP:journals/ijcv/LoyXG10}
\BIBentryALTinterwordspacing
C.~C. Loy, T.~Xiang, and S.~Gong, ``Time-delayed correlation analysis for
  multi-camera activity understanding,'' \emph{Int. J. Comput. Vis.}, vol.~90,
  no.~1, pp. 106--129, 2010. [Online]. Available:
  \url{https://doi.org/10.1007/s11263-010-0347-5}
\BIBentrySTDinterwordspacing

\bibitem{hinton2002stochastic}
G.~Hinton and S.~T. Roweis, ``Stochastic neighbor embedding,'' in \emph{NIPS},
  vol.~15.\hskip 1em plus 0.5em minus 0.4em\relax Citeseer, 2002, pp. 833--840.

\bibitem{zheng2015scalable}
L.~Zheng, L.~Shen, L.~Tian, S.~Wang, J.~Wang, and Q.~Tian, ``Scalable person
  re-identification: A benchmark,'' in \emph{Proceedings of the IEEE
  international conference on computer vision}, 2015, pp. 1116--1124.

\bibitem{wei2018person}
L.~Wei, S.~Zhang, W.~Gao, and Q.~Tian, ``Person transfer gan to bridge domain
  gap for person re-identification,'' in \emph{Proceedings of the IEEE
  conference on computer vision and pattern recognition}, 2018, pp. 79--88.

\end{thebibliography}
%

%




\end{document}